\documentclass{article}

\PassOptionsToPackage{authoryear, round}{natbib}

\usepackage[main, final]{neurips_2026}

\usepackage[utf8]{inputenc} 
\usepackage[T1]{fontenc}    
\usepackage{hyperref}       
\usepackage{url}            
\usepackage{booktabs}       
\usepackage{amsfonts}       
\usepackage{nicefrac}       
\usepackage{microtype}      
\usepackage{xcolor}         

\usepackage{amssymb}            
\usepackage{mathtools}          
\usepackage{mathrsfs}           
\usepackage{graphicx}           
\usepackage{subcaption}         
\usepackage[space]{grffile}     
\usepackage{url}                
\usepackage{lipsum}             
\usepackage{float}  
\usepackage{multirow}

\usepackage{amsthm}

\theoremstyle{plain}
\newtheorem{theorem}{Theorem}[section]
\newtheorem*{theorem*}{Theorem}

\theoremstyle{definition}

\newtheorem{assumption}[theorem]{Assumption}

\theoremstyle{remark}
\newtheorem{remark}[theorem]{Remark}

\usepackage{caption} 

\title{
On Distributional Reinforcement Learning \\in Chaotic Dynamical Systems 
}

\author{%
  James Rudd-Jones \\
  Centre for Artificial Intelligence\\
  Department of Computer Science\\
  University College London\\
  London, UK \\
  \texttt{james.rudd-jones.22@ucl.ac.uk} \\
  \And
  Mirco Musolesi \\
  Centre for Artificial Intelligence\\
  Department of Computer Science\\
  University College London\\
  London, UK \\
  Department of Computer Science and\\
  Engineering, University of Bologna\\
  Bologna, Italy \\
  \texttt{m.musolesi@ucl.ac.uk} \\
  \AND
  María Pérez-Ortiz \\
  Centre for Artificial Intelligence\\
  Department of Computer Science\\
  University College London\\
  London, UK \\
  \texttt{maria.perez@ucl.ac.uk} \\
}

\begin{document}

\maketitle

\begin{abstract}
Chaotic dynamical systems pose a fundamental challenge for Reinforcement Learning (RL): exponential sensitivity to initial conditions induces high-variance bootstrap targets and poorly conditioned gradient updates. 
Chaotic dynamics arise across scientific and engineering domains, from fluid flows and climate systems to multi-agent systems, where reliable learning is highly desirable.
Standard RL methods optimise expected returns through scalar value functions, implicitly averaging over diverging trajectories and entangling trajectory level instability with the learning objective.
We show that under mild statistical stability assumptions, the return distribution evolves more regularly than individual trajectories when measured under the $1$-Wasserstein metric, yielding a smoother distributional Bellman objective.
By aligning optimisation with this measure level structure, distributional RL provides better conditioned learning.
We offer a principled explanation for the advantages of distributional methods in chaotic systems and the geometries of RL objectives under chaos.
\end{abstract}

\section{Introduction}
In many applications of Reinforcement Learning (RL), chaos as itself is oft ignored and is instead conflated with stochasticity, but this framing obscures some of its important structure.
In fact, unlike stochastic noise, chaos is typically deterministic in nature, arising from well-defined underlying dynamical rules, yet unpredictable because of amplification of infinitesimal perturbations over time \citep{strogatz2024nonlinear}.
The compounded errors lead to multiple qualitatively distinct long term outcomes and sets of possible system states that form intricate self-similar geometric structures rather than smooth landscapes \citep{mandelbrot2004fractals}.
As a result, averaging returns, a standard assumption in RL, can be fundamentally misleading.
The intersection of RL and chaotic dynamics is broad and increasingly prominent, spanning multi-agent systems \citep{sato2002chaos, rudd2026rlc}, the stabilisation of financial systems \citep{rigatos2019nonlinear}, turbulence reduction in aero- and hydrofoil flows \citep{montala2025deep}, climate policy design \citep{rudd2025crafting, rudd2025multi}, and synchronisation of power grids \citep{eroglu2017synchronisation, halekotte2021transient}.
However, the characteristics of chaotic systems render many traditional RL assumptions invalid, including smoothness in value landscapes and stable gradient updates \citep{boccaletti2000control, wang_fractal_2023, wang_mollification_2024}.

We start with a vignette for the reader: consider two states separated by a singular floating point numerical precision unit in a chaotic system. 
Under identical open-loop policies, one trajectory reaches a goal with high reward while the other diverges into a failure basin. 
Standard RL collapses these outcomes into a single scalar value, producing a misleading target that corresponds to no realisable behaviour.
In chaotic systems, infinitesimal errors in state estimation or action selection amplify exponentially over time, even tiny differences in state can lead to vastly different future trajectories which undermines reliable long-term prediction and control. 
This ``butterfly effect" \citep{lorenz1972predictability} creates a hostile optimisation landscape as we lack smoothness in state for value function representation and can receive a large variance of returns when trying to directly optimise policies. 
Traditional methods typically struggle because the bounds of value prediction error explode rapidly, rendering long-horizon planning inextricably noisy \citep{young2024enhancing}.

We study the mismatch between chaotic dynamics and the optimisation objective used in standard RL, which typically maximises expected return via scalar value functions. 
In chaotic systems, due to the exponential divergence of trajectories, long-horizon returns become highly sensitive to initial conditions. 
Averaging over such trajectories produces irregular bootstrapped targets, leading to sharp curvature and severe gradient scaling issues in the optimisation landscape.
Although individual trajectories diverge, the evolution of probability measures describing long-term system behaviour is often statistically regular \citep{ruelle1976measure,bowen2008ergodic}. 
Motivated by this observation, we study RL through the lens of return distributions and Wasserstein geometry: under mild statistical stability assumptions, the return distribution varies more regularly than trajectory level quantities, yielding a smoother optimisation objective and providing a principled explanation for the potential effectiveness of distributional RL in chaotic environments.
Our contributions can be summarised as follows:
\begin{itemize}
    \item We prove that for chaotic dynamics the distributional RL objective is smoother than the expectation based objective: under mild statistical stability assumptions the return distribution is Lipschitz continuous in the $1$-Wasserstein metric even when trajectories diverge exponentially.
    \item We empirically analyse optimisation geometry in chaotic environments and show that distributional objectives produce smoother loss landscapes and lower variance one-step targets, in turn enforcing bounded gradient norms.
    \item We show that distributional Q-learning methods can improve converged episodic return for some canonical control in chaos experiments over a non-distributional approach.
\end{itemize}

\section{Related Work}

\textbf{Model-Free Reinforcement Learning for Chaotic Systems.}
The optimisation landscape in chaotic environments is notoriously difficult: \citet{wang_fractal_2023} demonstrated that the value functions and policy optimisation landscape in chaotic systems is non-smooth, rendering standard gradient based methods unstable or inefficient.  
This behaviour stems from the non-smoothness in state caused by sensitivity to initial conditions and exponential divergence of trajectories.
To address these irregularities, \citet{wang_mollification_2024} analysed the mollification effects of stochastic policies, showing that the combination of a fractal value function with the Gaussian policy kernel effectively smooths out the high frequencies. 
This theoretical insight explains why policy gradient methods can succeed in chaotic domains as the policy acts as an inherent low-pass filter.
However, in some environments this smoothing impacts the global objective, causing non-distributional RL to fail \citep{wang_mollification_2024}.

\textbf{Distributional Reinforcement Learning for Chaotic Systems.}
A canonical example of distributional RL to chaotic control is the autonomous navigation of stratospheric balloons studied by \citet{bellemare2020autonomous}.
In this setting, the atmospheric wind field forms a chaotic dynamical system and horizontal control is achieved solely by adjusting altitude to exploit different wind currents. 
The authors employed QRDQN \citep{dabney2018distributional}, a distributional RL algorithm that exceeded non-distributional approaches, attributed in part to the multimodal return distribution induced by chaotic winds. 
Our work can be seen as providing theoretical justification for the empirical success of distributional RL observed by \citet{bellemare2020autonomous} for control under chaos.


\textbf{Smoothness in Distributional Reinforcement Learning.}
While standard RL optimises the expected return, distributional RL models the full return distribution \citep{bellemare_distributional_2017}. 
Beyond improved representational capacity, prior work shows that distributional methods can yield more favourable optimisation properties than expectation based approaches \citep{sun2022does}. 
In particular, the distributional Bellman operator is Lipschitz continuous under metrics such as the Wasserstein distance \citep{bellemare_distributional_2017} and the Cram\'er distance \citep{rowland2018analysis}, even when the expected Bellman operator may fail to contract. 
These results suggest that distributional RL implicitly smooths the optimisation landscape by aggregating information across the return distribution. 
Such smoothing is reminiscent of the mollification observed in policy gradient methods and may help explain the empirical robustness of distributional RL in high variance settings such as chaotic environments. 
However, to the best of our knowledge, the interaction between distributional RL and chaotic dynamics has not previously been analysed.

\section{Chaotic Dynamics and their Implications for Reinforcement Learning}
\vspace{-10pt}

\begin{figure}[htbp] 
    \captionsetup{skip=-5pt}
    \centering
        \includegraphics[width=1\linewidth]{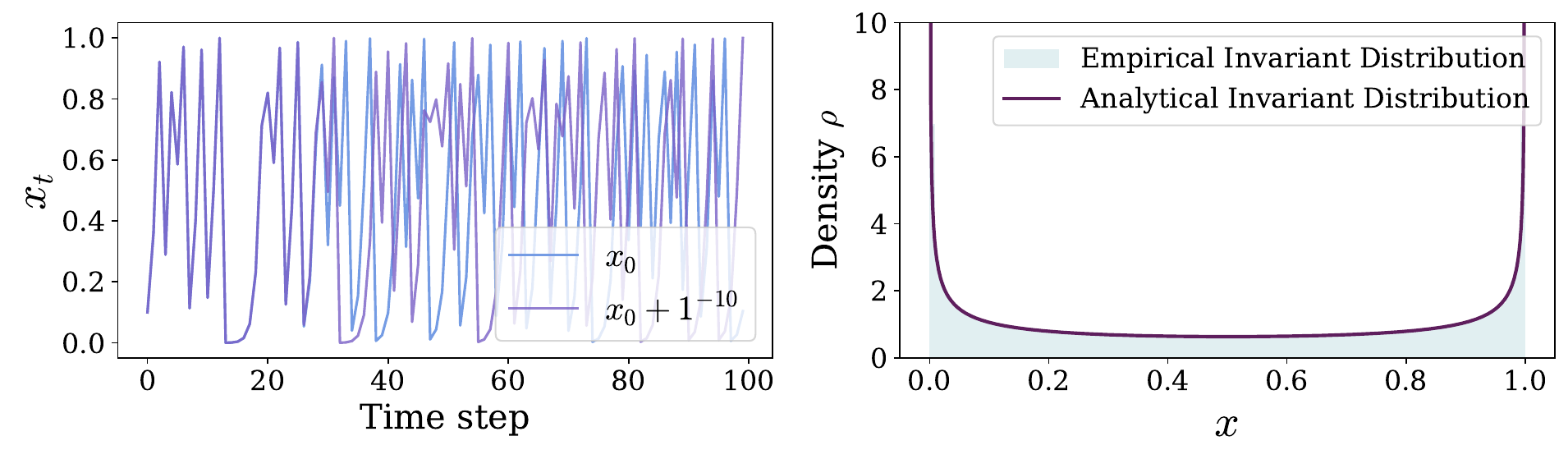}
    \caption{Left: Two trajectory realisations. Right: Empirical and analytical invariant distributions.}
    \label{fig:invariant_figure_together}
    \vspace{-5pt}
\end{figure}


In the section we provide an overview of our study, introducing chaotic dynamics and their impact on Reinforcement Learning (RL). 
We discuss how chaos is typically approached from a probabilistic perspective, drawing parallels with learning the full distribution of returns in distributional RL.
We posit that distributional RL mitigates several failure modes that arise in chaotic environments, where expectation-based value learning often struggles. 
However, we observe that distributional RL may not be sufficient on its own to tackle the full range of challenges considered in all scenarios.

\subsection{What Does Chaos Mean Mathematically?}
Non-linear dynamics frequently exhibit chaotic behaviour under certain parameters; chaos is a ubiquitous phenomenon across many natural processes \citep{strogatz2024nonlinear}.
Formally, a (continuous-time) dynamical system is defined on a phase space $M$, a smooth manifold where every possible state of the system corresponds to a unique point $s \in M$.
The system evolves according to a flow $\phi_t : M \rightarrow M$ such that $s(t) = \phi_t(s_0)$ describes the trajectory from initial condition $s_0$ given time $t$.
Chaos is characterised by an extreme sensitivity to initial conditions \citep{strogatz2024nonlinear}.
For two infinitesimally close initial states $s_0$ and $s_0 + \delta s_0$ the distance between their trajectories at time $t$ separates exponentially as $\| s'(t) - s(t) \| \approx e^{\lambda t} \|\delta s_0\|$, where $s(t)$ is the trajectory starting at $s_0$ and $s'(t)$ the trajectory starting at $s_0 + \delta s_0$, quantified by the Lyapunov exponent $\lambda$ \citep{wolf1985determining, shimada1979numerical}.
With maximal Lyapunov exponent $\lambda_{\max} = \lim_{t \rightarrow \infty} \lim_{\| \delta s_0 \| \rightarrow 0} \frac{1}{t} \ln \frac{\| \delta s(t) \|}{\| \delta s_0 \|}$.
Systems with $\lambda_{\max} > 0$ exhibit exponential error amplification, are typically regarded chaotic, imply that long term prediction becomes exponentially sensitive to initial conditions despite deterministic governing equations.
The reciprocal of the maximal Lyapunov exponent, known as Lyapunov time $T_\lambda \approx 1/\lambda_{\max}$, represents the horizon beyond which state predictions become effectively meaningless due to the exponential amplification of infinitesimal errors.

\subsection{Why Distributions are More Stable than Trajectories Under Chaotic Dynamics}
Dissipative dynamical systems typically evolve toward an attractor, an invariant subset of the phase space \citep{strogatz2024nonlinear}.
While simple systems admit fixed points or limit cycles, chaotic systems often possess strange attractors, which exhibit fractal geometry and thus a non-integer Hausdorff dimension \citep{ruelle1976measure}.  
Interestingly, although trajectories diverge exponentially, the strange attractor is a bounded set in phase space that confines trajectories, within which they exhibit chaotic behaviour.

Strange attractors admit invariant probability measures, often Sinai-Ruelle-Bowen (SRB) measures, which characterise the system's long term statistical behaviour \citep{sinai1972gibbs, ruelle1976measure, bowen2008ergodic}. 
Modelling this invariant distribution is substantially more robust than predicting individual trajectories as beyond Lyapunov time arbitrarily small errors are exponentially amplified, making long-horizon accuracy impossible \citep{jiang_training_2023, schiff_dyslim_2024}. 
In contrast, learning the geometry and statistics of the attractor captures physically valid long term behaviour.

The Logistic Map, a canonical example of discrete time chaos illustrates this distinction, defined as $s_{t+1} = ms_t(1-s_t)$ where $m$ adjusts the mapping's chaotic nature.
Infinitesimally close initial conditions quickly decorrelate (Figure \ref{fig:invariant_figure_together} Left), yet the empirical distribution of visited states converges to a stable invariant density (Figure \ref{fig:invariant_figure_together} Right). 
Thus, although trajectories diverge exponentially and are individually unpredictable, the map admits a stable invariant measure, so that long-run time averages and distributional properties converge.


\subsection{Challenges of Chaotic Transitions in Reinforcement Learning}
Recent findings by \citet{wang_fractal_2023} demonstrate that in chaotic environments, value functions and consequently the policy gradient landscape exhibit fractal structure.
In chaotic systems, because the return is defined through long-horizon rollouts, the state-action value function $Q(s,a)$ becomes highly sensitive to small changes in initial conditions and may be non-smooth with respect to state.

Such non-smoothness directly impacts gradient based optimisation. 
Temporal difference (TD) methods rely on bootstrapped targets of the form $r + \gamma \max_{a'} Q(s',a')$, irregularities in $Q$ propagate through the Bellman recursion. 
Small state differences produce large target variations, increasing gradient variance and inducing sharp curvature in the loss landscape. 
In these regimes, local gradient information provides limited guidance for global progress, optimisation is highly sensitive to step size and initial conditions, and convergence may stall or oscillate near the optimum.
These issues arise from properties of the underlying dynamical system.

Distributional RL \citep{bellemare_distributional_2017} reflects this perspective by modelling the full return distribution rather than its expectation. 
The distributional Bellman operator is a contraction in the $1$-Wasserstein metric, providing a well defined fixed point in a probability space \citep{bellemare_distributional_2017}.
By optimising distances between return distributions, rather than scalar value errors, distributional methods shift the objective to a metric space with smoother geometry, potentially yielding a better conditioned optimisation landscape.
In this sense, treating chaotic transitions probabilistically principally aligns the learning objective with the intrinsic instability of the dynamics.

\section{How is Distributional Reinforcement Learning Beneficial for Chaotic Dynamics?}



We model the environment as a time-homogeneous Markov Decision Process (MDP) defined by the tuple $\left < \mathcal S, \mathcal A, P, R, \mu_0, \gamma \right >$ where $\mathcal S$ and $\mathcal A$ denote the continuous state and finite action spaces, $P : \mathcal S \times \mathcal A \rightarrow \Delta (\mathcal S)$ the transition dynamics, $R : \mathcal S \times \mathcal A \rightarrow \mathbb R$ the reward function, $\mu_0$ the initial state distribution, and $\gamma \in (0,1)$ the discount factor.
At each timestep $t$ an agent selects action $a_t \sim \pi(\cdot \mid s_t)$, the MDP transitions to the next state $s_{t+1} \in \mathcal S$ with probability $P(s_{t+1}|s_t, a_t)$ and in doing so receives reward $R(s_t, a_t)$. 
Actions are selected by the policy $\pi : \mathcal S \rightarrow \Delta(\mathcal A)$.
We maximise the discounted return  $Z^\pi=\sum_{t=0}^\infty \gamma^t R(s_t, a_t)$.
The state value and state-action value functions:
\begin{align}
    V^\pi(s) &= \mathbb E_{a \sim \pi,\, s' \sim P} \left[ R(s,a) + \gamma V^\pi(s') \right], \quad \quad \; \; \; \; \;V^\pi(s) = \mathbb E_\pi \left[ Z^\pi \mid s_0 = s \right], \\
    Q^\pi(s,a) &= \mathbb E_{s' \sim P,\, a' \sim \pi} \left[ R(s,a) + \gamma Q^\pi(s',a') \right], \quad Q^\pi(s,a) = \mathbb E_\pi \left[ Z^\pi \mid s_0 = s, a_0 = a \right].
    \label{eqn:value_funcs}
\end{align}
Non-distributional RL estimates expectations of returns; in contrast, Distributional RL models the full return distribution. 
Let $Z(s,a)$ denote the random return obtained from state action pair $(s,a)$. 
The distributional Bellman equation is $Z(s,a) \overset{D}{=} R(s,a) + \gamma Z(S', A')$, where $S' \sim P(\cdot \mid s,a)$, $A' \sim \pi(\cdot \mid S')$, and $\overset{D}{=}$ denotes distribution equality.
Distributional RL minimises distances between return distributions under metrics such as Wasserstein, yielding contraction properties of the distributional Bellman operator, ensuring stable convergence \citep{bellemare_distributional_2017}.

\subsection{A Distributional Objective with Wasserstein can be Smoother under Chaos}
We argue that in chaotic dynamical systems, expectation based objectives lead to challenging gradient updates and highly irregular value functions \citep{wang_fractal_2023}. 
In contrast, a distributional objective using the Wasserstein metric yields smoother learning targets.
To clarify, \textit{we do not claim that the true return distribution is smooth in state}. 
Instability in expectation based methods occurs at the level of individual trajectories, even though statistical properties of chaotic systems can remain stable under small perturbations. 
This distinction motivates distributional RL, comparing return distributions under Wasserstein induces metric regularity absent from expectation based approaches. 
Consequently, although the mapping $s \mapsto Z^\pi(s,a)$ may be highly irregular at individual points, it can still be Lipschitz continuous as a mapping into the space of probability measures endowed with $W_1$, provided the transition dynamics are statistically stable.
We formalise this result under the assumptions below: 

\begin{assumption}[One-Step Lipschitz Dynamics]
\label{assumption:pathwise_divergence}
    Let $\pi$ be a Lipschitz continuous policy. 
    We assume the closed-loop one-step dynamics $s_{t+1} = f(s_t,\pi(s_t))$ are $K_f$-Lipschitz continuous:
    \begin{equation}
        \|f(s_1, \pi(s_1))-f(s_2, \pi(s_2))\| \leq K_f \|s_1-s_2\|,
    \end{equation}
    with $K_f>1$ in chaotic regimes.
    This assumption quantifies how perturbations in the state propagate through the dynamics.  
    In chaotic regimes ($K_f > 1$), small differences in initial states lead to exponential trajectory divergence.
    Note that traditional Q-learning relies upon a greedy policy: a discontinuous step function and thus not Lipschitz continuous.
    By assuming Lipschitz continuity (e.g., Softmax policy or Gaussian policy as in PPO) we establish a lower bound on the instability.
\end{assumption}
\vspace{-5pt}
\begin{assumption}[Reward Regularity]
\label{assumption:reward_regularity}
For all $a \in \mathcal A$ the reward function is $K_R$-Lipschitz continuous in state and rewards are bounded:
    \begin{equation}
        |R(s_1,a) - R(s_2, a)| \leq K_R \|s_1 - s_2\|.
    \end{equation}
    This assumption ensures that small perturbations in the state lead to controlled changes in reward.  
    Such smoothness is typical in problems where rewards correspond to continuous cost functions.
\end{assumption}
\vspace{-5pt}
\begin{assumption}[Wasserstein Statistical Stability of One-Step Transitions]
\label{assumption:statistical_stability}
    For each $a \in \mathcal A$ the transition distribution $P(\cdot | s,a)$ is $K_P$-Lipschitz continuous with respect to the Wasserstein-1 metric $W_1$. 
    That is for any $a$:
    \begin{equation}
        W_1 \left(P(\cdot | s_1,a),P(\cdot | s_2,a) \right) \leq K_P \| s_1 - s_2 \|.
    \end{equation}
\end{assumption}
\vspace{-5pt}
This assumption captures statistical regularity of the dynamics: although nearby trajectories may diverge, the resulting state distributions remain close in Wasserstein distance.  
While purely deterministic and perfectly observed chaotic systems may violate Assumption \ref{assumption:statistical_stability}, most pragmatic RL problems introduce effective stochasticity through process noise, exploration noise, partial observability, state discretisation, sensor uncertainty, etc. 
Under these conditions the transition kernel becomes non-degenerate and varies smoothly in Wasserstein distance. 
Fully deterministic chaotic systems represent a special limiting case, whereas statistically stability is the norm in practical control \citep{viana1997stochastic}.
\begin{assumption}[Discount Condition for Distributional Stability]
We assume that: 
 \label{assumption:discounting}
    \begin{equation}
        \gamma K_P < 1
    \end{equation}
\end{assumption}
\vspace{-5pt}
This condition ensures that the Bellman operator is contractive under distributional perturbations.  
Intuitively, discounting prevents transition sensitivity from amplifying errors over long horizons.
\begin{theorem}[One-Step Sensitivity of Scalar Returns]
\label{theorem:exponential_onestep}
    Under Assumptions \ref{assumption:pathwise_divergence} and \ref{assumption:reward_regularity}, for any closed-loop policy $\pi$, the truncated return map $s \mapsto G_T(s)$ is Lipschitz ($\mathrm{Lip}(G_T)$) with constant:
    \vspace{-6pt}
    \begin{equation}
        \mathrm{Lip}(G_T)\;\le\;K_R\sum_{t=0}^{T-1}(\gamma K_f)^t
        \;=\;
        K_R\cdot\frac{(\gamma K_f)^T-1}{\gamma K_f-1}\quad (\gamma K_f\neq 1).
    \end{equation}
    In particular, if $\gamma K_f>1$, then $\mathrm{Lip}(G_T)$ grows exponentially in $T$.
    Proof in Appendix \ref{sec:appendix1}.
\end{theorem}
\vspace{-8pt}
In the theorem above we show that one-step samples under chaotic dynamics can explode.

\begin{theorem}[$W_1$-smoothness of the return distribution under statistical stability]
\label{theorem:smoothness_wasserstein}
    Under Assumptions \ref{assumption:reward_regularity} - \ref{assumption:discounting}, for any fixed initial action $a$ the following holds:
    \begin{equation}
        W_1\!\left(Z^\pi(s_1,a),Z^\pi(s_2,a)\right)\le \frac{K_R}{1-\gamma K_P}\,\|s_1-s_2\|.
    \end{equation}
    Proof in Appendix \ref{sec:appendix2}.
\end{theorem}
\vspace{-5pt}

Theorem \ref{theorem:smoothness_wasserstein} establishes that, under Assumption \ref{assumption:statistical_stability}, the true return distribution $Z^\pi(s,a)$ varies Lipschitz continuously with respect to state when measured under $1$-Wasserstein.
We emphasise that the expectation of a Lipschitz distribution is itself Lipschitz.
Thus Theorem \ref{theorem:smoothness_wasserstein} does not imply that expectation based value functions are non-smooth.
Rather, when viewed as mappings into the space of probability measures equipped with $W_1$, return distributions vary in a stable manner across nearby states.
The real advantage of distributional RL is that the distributional optimisation objective has better conditioning than scalar TD errors.
Chaotic systems may exhibit exponential divergence at the trajectory level, but their induced evolution of probability measures can remain stable.
Distributional RL leverages this measure level regularity rather than suppressing trajectory level chaos.
In Appendix \ref{sec:appendix3} we discuss in greater detail how distributional RL still copes when \ref{assumption:statistical_stability} fails.

\begin{remark}
    It is crucial to note that Theorem \ref{theorem:exponential_onestep} and \ref{theorem:smoothness_wasserstein} evaluate the exact same closed-loop system, yet yield vastly different stability guarantees. 
    Theorem \ref{theorem:exponential_onestep} relies on the trajectory expansion rate $K_f$. 
    In chaotic systems, local exponential divergence guarantees that $K_f > 1$, forcing the scalar return to explode.
    Conversely, Theorem \ref{theorem:smoothness_wasserstein} relies on $K_P$, the Lipschitz constant of the transition kernel in the $1$-Wasserstein metric. 
    A defining feature of chaotic strange attractors is that while individual trajectories diverge exponentially ($K_f > 1$), their invariant probability measures are statistically stable. 
    Therefore, it is consistent for a chaotic system to simultaneously exhibit $K_f > 1/\gamma$ (breaking scalar value smoothness) and $K_P < 1/\gamma$ (preserving distributional smoothness). 
\end{remark} 

\begin{remark}
    If the environment is fully observed and deterministic, then Assumption \ref{assumption:statistical_stability} reduces to the dynamics Lipschitz constant, so $K_P = K_f$. 
    In this regime $\gamma K_P < 1$ reduces to $\gamma K_f <1$ which fails in chaotic systems ($K_f >1$) and $W_1$-smoothness does not improve over one-step sensitivity.
    The regime where distributional smoothing becomes meaningful is when the effective transition kernel is non-degenerate which is very common in pragmatic RL. 
    In such cases the kernel can be statistically stable in $W_1$ with $K_P \ll K_f$ even when trajectories exhibit strong local expansion.
\end{remark}

\section{Empirical Analysis}
Above we have given a theoretical proof for the smoothness of the distributional RL objective under chaos.
In the following section we experientially support these claims.
We make a distinction of two types of chaotic environments: (i) settings in which the agent directly influences and stabilises the underlying chaotic system, and (ii) settings in which the chaotic dynamics are intrinsic to the environment and cannot be eliminated, but nevertheless influence the state evolution and observed returns.
We experiment on both types to understand the benefits distributional methods bring. 

\begin{figure}[t]
    \captionsetup{skip=0pt}
    \centering
        \includegraphics[width=1\linewidth]{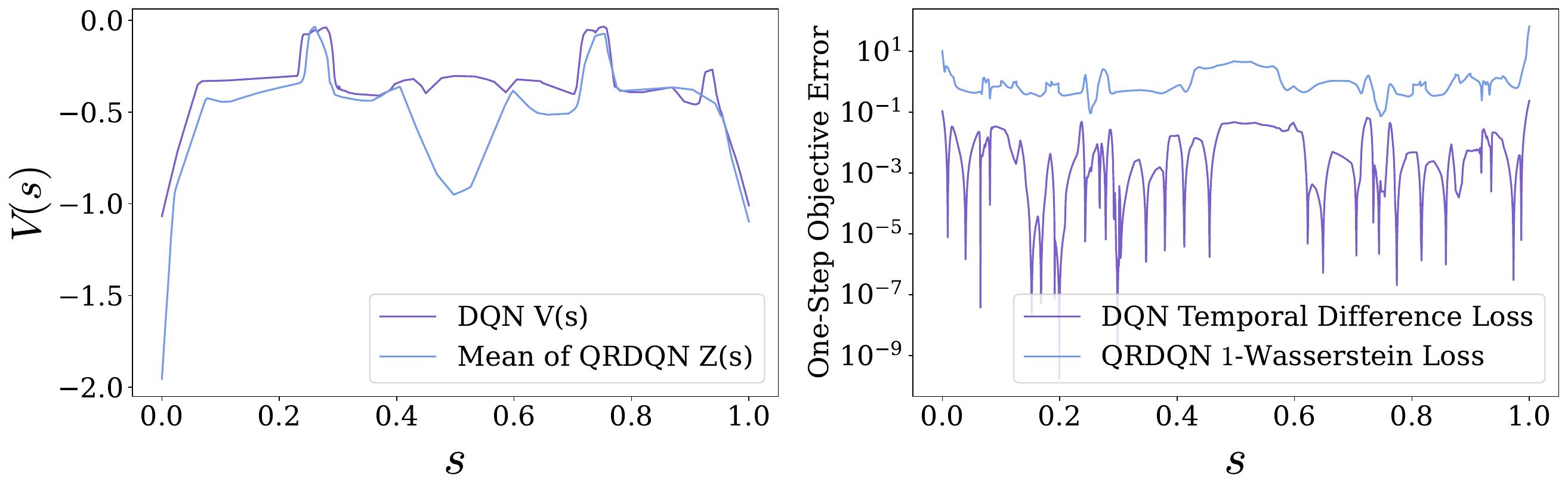}
    \caption{Left: $Q(s,a)$ for Action $1$ for the Logistic Map, we take the mean of the QRDQN value function for a $1$D representation. Right: Comparison of one-step errors between DQN and QRDQN. Figure data collected from agents after $3e8$ training steps.}
    \label{fig:logistic_map_landscape}
    \vspace{-10pt}
\end{figure}

\subsection{Experimental Settings}
\textbf{Environments.} For the former we look to control the Logistic Map, a canonical discrete time $1$D chaotic map.
Following \citet{gadaleta_optimal_1999} the Logistic Map control task is defined as $x_{t+1} = (m+a_t)x_t(1 - x_t)$, where $m=3.8$ is a parameter that adjusts the chaotic nature of the equation, $a\in[-0.1, 0.1]$ that perturbs the chaotic parameter, with arbitrary distance metric reward $r=\|x_t - x_{fp}\|^2$ to incentive an agent to stabilise the system to an unstable $1$-periodic fixed point $x_{fp} = \frac{m-1}{m}$. 
We also use the Ikeda Map \citep{ikeda1979multiple}, a discrete time $2$D chaotic map defined as:
\begin{align}
    x_{t+1} &= 1 + (u + a_t)(x_t \cos \xi_t - y_t \sin \xi_t) \\
    y_{t+1} &= (u + a_t)(x_t \sin \xi_t + y_t \cos \xi_t), 
\end{align}
where $\xi_t = k - \frac{p}{1 + x_t^2 + y^2_t}$, $u$ an equation parameter (not control term), $a \in [-0.1, 0.1]$, the $2$D state vector $s_t = [x_t, y_t]^\intercal \in \mathbb R^2$, with arbitrary distance metric reward $r=\|s_t - s_{fp}\|^2_2$ that incentives an agent to stabilise the system to a $1$-periodic fixed point $s_{fp} = [0.533, 0.247]$.
Chaotic behaviour is typical when $u=0.9$, $k=0.4$, and $p=6.0$.

For the latter we use two continuous time chaotic flow based environments in which a ``swimmer" must navigate towards a desired area of the state.
In a similar manner to \citet{bellemare2020autonomous} the swimmer is under-actuated, unable to move directly toward the goal, and thus must \textit{understand and work with the chaos}.
For a $2$D scenario we use the canonical Double Gyre flow, that simulates simplified ocean currents, and a swimmer therein \citep{sulalitha2017quantifying}, defined as:
\begin{align}
    \dot x &= - \pi A \sin (\pi f(x,t)) \cos (\pi y) \\
    \dot y &= \pi A \cos (\pi f(x,t)) \sin (\pi y) \frac{\partial f}{\partial x},
\end{align}
where $f(x,t)= a(t) x^2 + b(t)x$, $a(t) = \zeta \sin (\omega t)$, $b(t) = 1 - 2 \zeta \sin(\omega t)$, $x\in[0, 2], y\in[0,1]$, $A$ the flow magnitude, $\omega$ adjusts oscillation frequency, and $\zeta$ adjusts oscillation amplitude in $x$.
Setting $\zeta = 0$ creates a time homogeneous system that is time inhomogeneous otherwise.
We also use the $3$D Arnold–Beltrami–Childress (ABC) flow \citep{qin2023kind}, defined as:
\begin{align}
    \dot x &= A \sin z + C \cos y \\
    \dot y &= B \sin x + A \cos z \\
    \dot z &= C \sin y + B \cos x.
\end{align}
For both environments we set a reward of $r(s) = -0.01 + 10 \cdot \exp \left(-d(s)^2 /2 \epsilon ^2 \right)$, where $d(s) = \|s - s_\text{goal}\|$, to give a smooth Gaussian peak within $\epsilon$ of the goal state to ensure Assumption \ref{assumption:reward_regularity}.

\textbf{Algorithms.} For our distributional RL algorithm we chose QRDQN \citep{dabney2018distributional} that uses a quantile regression loss which is equivalent to minimising the 1-Wasserstein distance \citep{dabney2018distributional}.
The non-distributional method of choice was DQN \citep{mnih2013playing}.

\textbf{Description of the Experiments.} We primarily focus on plotting the one-step error between distributional and non-distributional methods to compare smoothness in state.
To explicitly evaluate optimisation conditioning, we calculate the $L_2$ gradient norm ($||\nabla_\theta \mathcal{L}||_2$) and the relative local gradient variance across the state space for the $1$D Logistic Map.

\begin{figure}[t]
    \captionsetup{skip=0pt}
    \centering
        \includegraphics[width=1\linewidth]{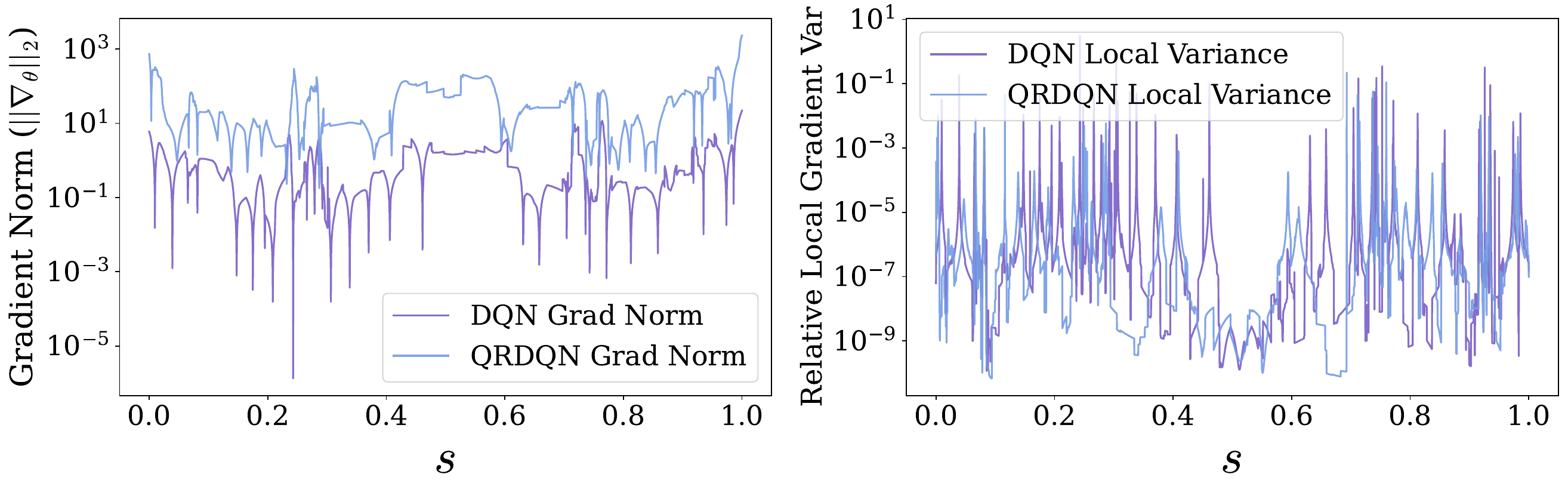}
    \caption{Left: Gradient norms ($||\nabla_\theta||_2$) of DQN and QRDQN for the Logistic Map. Right: Local gradient variance of DQN and QRDQN. Figure data collected from agents after $3e8$ training steps.}
    \label{fig:logistic_map_grad_landscape}
    \vspace{-10pt}
\end{figure}

\subsection{Results}
In Figure \ref{fig:logistic_map_landscape} we compare the landscapes of the value function and one-step error between DQN and QRDQN.
On the left it is clear that both learn a similar representation of the environment, where we have flattened the QRDQN distribution by taking the mean over the quantiles.
On the right houses some important results, the DQN temporal difference loss exhibits high variance and frequent sharp spikes, whereas the QRDQN $1$-Wasserstein loss remains comparatively smoother and more stable across states: suggesting that the distributional objective provides an improved optimisation signal.
In Appendix \ref{sec:appendix4} we show the full QRDQN return distribution for the Logistic Map.
The downstream impacts of this better behaved optimisation signal are seen in Figure \ref{fig:logistic_map_grad_landscape}.
On the right, the local gradient variance remains complex for both DQN and QRDQN; the value functions are inherently non-smooth and thus the neural network's gradients will always show localised variance spikes when trying to fit these high-frequencies \citep{wang_fractal_2023}.
However, on the left the gradient norm for QRDQN is bounded and maintains a consistent, non-vanishing range.
QRDQN gradient norms are naturally larger since the neural network outputs $N$ distinct quantiles rather than a singular scalar, but what is important is that the range of the norms is bounded. 
DQN gradient norms fluctuate over a larger range, even reaching a minimal value of $\sim 10^{-7}$. 

In further support, we display the objective for the Double Gyre flow in Figure \ref{fig:double_gyre_landscape}.
Interestingly in Figure \ref{fig:double_gyre_landscape} one can see the goal state in the top right corner for both algorithms, but for QRDQN the one-step loss is much smoother over the state space with a much reduced error range.  
Since most trajectories accumulate similar step penalties before either reaching the goal or failing, the induced return distributions differ only minimally across nearby states. 
The Wasserstein objective correctly captures this measure level similarity, yielding a smooth and well conditioned landscape. 
In contrast, the scalar TD loss of DQN exhibits substantial variability due to bootstrap amplification of local estimation noise. 
Similar plots for the Ikeda Map can be found at Figure \ref{fig:ikeda_map_landscape} in Appendix \ref{sec:appendix4}.

\begin{figure}[t]
    \captionsetup{skip=0pt}
    \centering
        \includegraphics[width=1\linewidth]{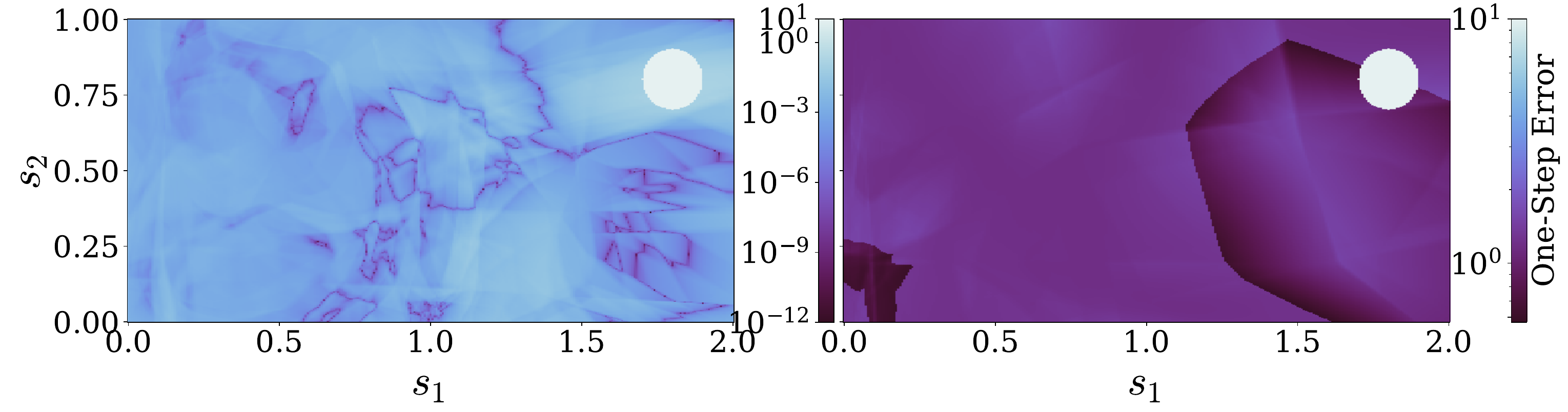}
    \caption{Double Gyre Flow objective landscapes. Left: DQN, Right: QRDQN. Maximal error has been clipped from $10^3$ to $10^1$ for clarity. Figure data collected from agents after $3e8$ training steps.}
    \label{fig:double_gyre_landscape}
    \vspace{-10pt}
\end{figure}

\subsection{Discussion}
While both DQN and QRDQN approximate value functions over chaotic dynamics, the optimisation geometry differs markedly. 
The fractal structure of the value function is unavoidable, as it reflects the intrinsic properties of the underlying chaotic dynamical system, and thus gradient variance is naturally also non-smooth. 
In Theorem \ref{theorem:exponential_onestep} we proved that for scalar returns in chaotic regimes the Lipschitz constant grows exponentially; the downstream consequence of which can be vanishing/exploding gradients.
Further, Theorem \ref{theorem:smoothness_wasserstein} proved that under statistical stability, the return distribution is $W_1$-Lipschitz continuous with a strict, finite bound. 
Figure \ref{fig:logistic_map_grad_landscape} corroborates these insights: DQN suffers from gradient collapse, QRDQN maintains a bounded non-vanishing gradient norm.
We emphasise that a smoother objective does not guarantee improved asymptotic performance; instead it improves conditioning of gradient updates and mitigates the impact of extreme bootstrap targets. 
\section{Benchmark Evaluation}
We have shown theoretically and empirically the smoother objective landscape that distributional RL provides in chaotic systems.
To finally exemplify our hypothesis, we compare distributional and non-distributional methods on the four aforementioned canonical chaotic control environments.

\subsection {Experimental Settings} 
We track episodic return over training time, with each environment having a maximum number of steps. 
In the Logistic and Ikeda Map environments, episodes may also terminate once the system stabilises within a specified error of the fixed point.
For the Double Gyre and ABC Flow environments, episodes terminate when the agent reaches the goal.
Algorithms are evaluated using six random seeds, and we plot mean episodic return across runs with $10\%-90\%$ quantile range.
Further, we include PPO \citep{schulman2017proximal} as a policy gradient comparison.
We understand these are fairly simple environments in state space size, however show that RL methods still struggle here.
All experiments are run on an Nvidia RTX $4090$ with each run taking $\approx 6$ hours, requiring $\approx 20$GB of VRAM.

\subsection{Results}

Figure \ref{fig:benchmark_results} compares DQN, QRDQN, and PPO across the four environments. 
In the first two environments (Logistic Map and Ikeda Map), the dynamics become less chaotic over time as they are effectively stabilised through the learning process. 
As a result, return distributions narrow and the advantage of modelling the full return distribution, as done in distributional RL, becomes less pronounced. 
We can see this in our results as QRDQN shows only marginal or inconsistent improvements over DQN. 
In contrast, in the latter environments (Double Gyre Flow and ABC Flow), strong chaotic behaviour persists throughout training. 
For the Double Gyre flow, QRDQN has an unstable early training period but eventually settles upon an optimal policy with a low variance between seeds.
For the ABC flow, QRDQN eventually learns an improved policy but is unable to reach the same maximal return as PPO.
In these settings, the ability of QRDQN to capture and learn from the full distribution of possible outcomes provides a clearer advantage, leading to more stable and superior performance. 
In all experiments, PPO is much more sample efficient reaching convergence in $~1/5$ of the time it takes the other two.
Further, the results from PPO set an upper bound for the maximal episodic return in all experiments but the Logistic Map.

\begin{figure}[t]
\captionsetup{skip=0pt}
    \centering
        \includegraphics[width=1\linewidth]{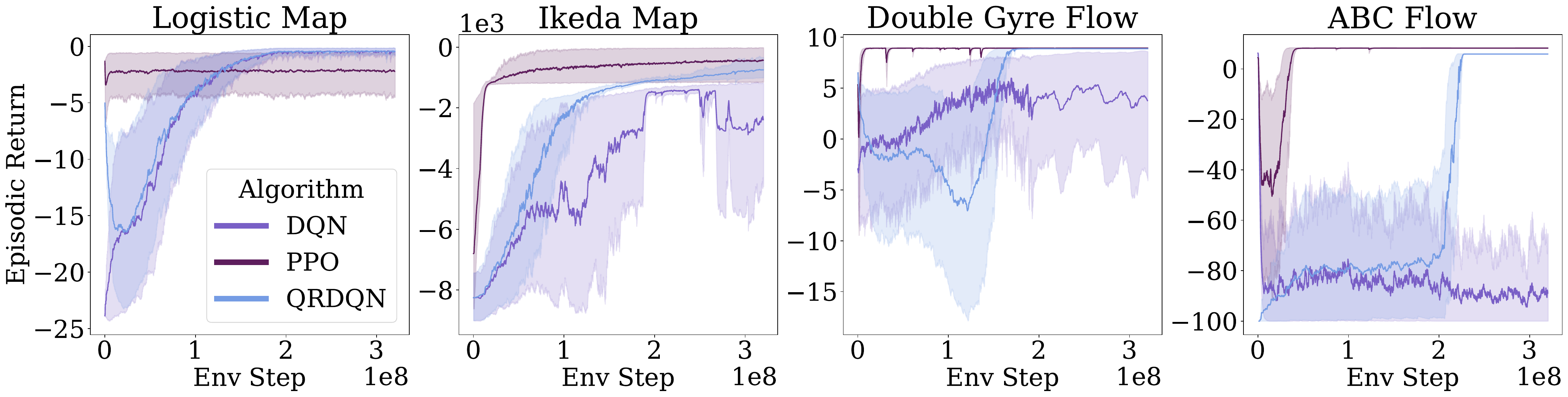}
    \caption{
    Each algorithm is run over six random seeds, we plot the mean episodic return across runs with shaded regions indicating the $10\%-90\%$ quantile range at each environment step.
    }
    \label{fig:benchmark_results}
    \vspace{-1pt}
\end{figure}

\subsection{Discussion}
Although in the Logistic and Ikeda Map the agent progressively suppresses chaotic behaviour by stabilising the system to a fixed point, even in the Double Gyre and ABC flow learning progressively regularises the experienced dynamics.
As policies become more consistent, the state visitation distribution concentrates on narrower regions of the attractor. 
In this sense, learning induces a form of ordering even when the underlying flow remains chaotic. 
However, return stochasticity persists due to sensitivity to initial conditions, noise, and perturbations. 
Distributional methods remain useful because they capture this residual multi-modality instead of collapsing outcomes into an expectation.

From Figure \ref{fig:benchmark_results} it is clear that PPO performs competitively across tasks. 
This aligns with the analysis of \citet{wang_mollification_2024} where Gaussian policies act as implicit mollifiers of fractal value landscapes. 
For Q-learning approaches the inherent challenge is that the policy is typically a step function, whereas in PPO we have a Lipschitz continuous policy.
Policy gradients smooth via action stochasticity, while distributional RL smooths via probability geometry (e.g. Wasserstein distance). 
These are distinct but related mechanisms, reinforcing the importance of optimisation smoothing in chaotic systems. 

In persistent chaos environments (Double Gyre and ABC Flow), QRDQN eventually reaches a similar episodic return as PPO but requires many more samples. 
This suggests distributional RL primarily improves optimisation stability (e.g. stabilised targets, reduced gradient variance) but has no effect on data efficiency. 
We therefore frame its benefit as improved conditioning and robustness in highly sensitive regimes for Q-learning based methods, but at least in the experiments we have used, the mollification effects of PPO are more impactful for sample efficiency and maximal return.

\section{Conclusion}
\textbf{Summary of the Contributions.} In this work we investigated how chaotic dynamics impact Q-learning based RL objectives. 
While chaos causes exponential trajectory sensitivity and non-smooth value landscapes, we show that return distributions can remain Lipschitz in Wasserstein distance under mild statistical stability assumptions. 
This measure level regularity helps explain why distributional RL produces better conditioned optimisation in chaotic environments.
Empirically, distributional methods reduce loss variability and stabilise training. 
Gains are modest when chaos is suppressed, but clearer when strong mixing persists. 
We do not claim universal superiority rather, distributional RL offers robustness in settings with unstable bootstrap targets for Q-learning based methodologies.

\textbf{Limitations.} Our theory assumes statistical stability in Wasserstein distance, which may fail near bifurcations or in perfectly deterministic systems. 
Although distributional bounds may break down, modelling multi-modal returns can still outperform averaging. 
Our experiments focus on discrete actions; extending the analysis to continuous control remains important. 
Finally, our results emphasise optimisation stability rather than asymptotic optimality or sample efficiency improvements.


\textbf{Future Work.} Future directions include extending theory to continuous actions or studying representation learning with attractor geometry.
For example, using distributional methods with policy gradient approaches would indicate whether the benefits of learning a distribution carry on over, perhaps the combination of mollification and learning a return distribution would be an optimal methodology for chaotic systems.
Distributional RL and policy gradient methods apply smoothing, but still learn a fractal value function, can one avoid this tough representation task?
Finally, recent work suggests alternative measures that better characterise chaotic dynamics \citep{finzi2026entropy}.

\begin{ack}
James Rudd-Jones is supported by grants from the UK EPSRC-DTP (Award 2868483)
\end{ack}

\bibliography{sample}
\bibliographystyle{plainnat}


\appendix



\section{Experimental Code}
Code for all the experiments will be released upon acceptance.
The chaotic control environments are part of a github package that will be linked here upon acceptance to ensure double blind is respected.
We'd like to thank \citet{thyng2016true} for creating \texttt{cmocean}, a great source of colour maps used throughout this paper!

\section{Proof of Theorem \ref{theorem:exponential_onestep}}
\label{sec:appendix1}

\begin{theorem*}[Exponential one-step sensitivity of scalar sample returns]
    Under Assumptions \ref{assumption:pathwise_divergence} and \ref{assumption:reward_regularity}, for any fixed action sequence, the truncated return map $s\mapsto G_T(s)$ is Lipschitz ($\mathrm{Lip}(G_T)$) with constant
    \begin{equation}
        \mathrm{Lip}(G_T)\;\le\;K_R\sum_{t=0}^{T-1}(\gamma K_f)^t
        \;=\;
        K_R\cdot\frac{(\gamma K_f)^T-1}{\gamma K_f-1}\quad (\gamma K_f\neq 1).
    \end{equation}
    In particular, if $\gamma K_f>1$, then $\mathrm{Lip}(G_T)$ grows exponentially in $T$.
\end{theorem*}

\begin{proof}
    Let $s_t$ and $\tilde s_t$ be trajectories that start from $s_0$ and $\tilde s_0$ under the same closed-loop policy $\pi$.
    By the closed-loop dynamics Lipschitz property, $\|s_t-\tilde s_t\|\le K_f^t\|s_0-\tilde s_0\|$ by induction.
    Then:
    \begin{equation}
        |G_T(s_0)-G_T(\tilde s_0)|
    \le
    \sum_{t=0}^{T-1}\gamma^t |R(s_t,\pi(s_t))-R(\tilde s_t, \pi(\tilde s_t))|
    \le
    \sum_{t=0}^{T-1}\gamma^t K_R\|s_t-\tilde s_t\|
    \le
    K_R\sum_{t=0}^{T-1}(\gamma K_f)^t\|s_0-\tilde s_0\|.
    \end{equation}
\end{proof}


\section{Proof of Theorem \ref{theorem:smoothness_wasserstein}}
\label{sec:appendix2}

\begin{theorem*}[$W_1$-smoothness of the return distribution under statistical stability]
    Under Assumptions \ref{assumption:reward_regularity} - \ref{assumption:discounting}, for any fixed initial action $a$ the following holds:
    \begin{equation}
        W_1\!\left(Z^\pi(s_1,a),Z^\pi(s_2,a)\right)\le \frac{K_R}{1-\gamma K_P}\,\|s_1-s_2\|.
    \end{equation}
\end{theorem*}

\begin{proof}
    Let $\mathcal T^\pi$ be the distributional Bellman operator:
    $(\mathcal T^\pi Z)(s,a)\overset{D}{=}R(s,a)+\gamma Z(S',A')$ with $S'\sim P(\cdot\mid s,a)$ and $A'\sim\pi(\cdot\mid S')$.
    Let us fix $a$ and define:
    \begin{equation}
        K := \sup_{s_1 \neq s_2} \frac{W_1(Z^\pi(s_1,a), Z^\pi(s_2,a))} {\|s_1-s_2\|}.
    \end{equation}
    Using standard coupling plus the fact that $W_1$ contracts under $\gamma$-scaling,
    \begin{equation}
        W_1(Z^\pi(s_1,a), Z^\pi(s_2,a)) \le |R(s_1,a) - R(s_2,a)| + \gamma W_1 \left(Z^\pi(S'_1, A'_1), Z^\pi(S'_2, A'_2)\right).
    \end{equation}
    Now couple $S'_1, S'_2$ using an optimal $W_1$ coupling for $P(\cdot \mid s_1, a)$ and $P(\cdot \mid s_2, a)$, so $\mathbb E\|S'_1 - S'_2\| \le K_P\|s_1 - s_2\|$.
    Conditioning on $(S'_1, S'_2)$ and using that $Z^\pi(\cdot, a)$ is $K$-Lipschitz with respect to $W_1$ (by the definition of $K$), the mixture stability property of $W_1$ implies:
    \begin{equation}
        W_1 \left(Z^\pi(S'_1, A'_1), Z^\pi(S'_2, A'_2)\right) \le \mathbb E \big[W_1(Z^\pi(S'_1, A'_1),Z^\pi(S'_2, A'_2))\big] \le K \mathbb E\|S'_1 - S'_2\|.
    \end{equation}
    Here we use the standard inequality that $W_1$ between mixtures is bounded by the expectation of $W_1$ under a coupling of the mixing variables.
    Assumption \ref{assumption:reward_regularity} ensures all terms are well defined.
    Therefore,
    \begin{equation}
         W_1(Z^\pi(s_1, a), Z^\pi(s_2, a)) \le K_R\|s_1 - s_2\| + \gamma KK_P\|s_1 - s_2\|.
    \end{equation}
    Divide by $\|s_1 - s_2\|$ and take the Supremum:
    \begin{equation}
        K \le K_R + \gamma K K_P,
    \end{equation}
    hence $K\le \frac{K_R}{1 - \gamma K_P}$ when $\gamma K_P<1$.
    Thus concluding the proof.

\end{proof}

\section{When Statistical Stability Fails}
\label{sec:appendix3}

We now consider the regime in which Assumption \ref{assumption:statistical_stability} fails, and clarify how distributional RL behaves in this setting in more detail.

\textbf{Failure of Statistical Stability.} Assumption \ref{assumption:statistical_stability} requires that the transition kernel varies Lipschitz continuously in $W_1$ with respect to the state. 
This can fail at bifurcation points or near basin boundaries.
Consider a system with two disjoint attractors separated by a stable manifold.
Two arbitrarily close states on opposite sides of the boundary evolve toward
different attractors:
\begin{equation}
    P(\cdot \mid s_1,a) = \delta_{\mathcal A_L}, \quad P(\cdot \mid s_2,a) = \delta_{\mathcal A_R},
\end{equation}
where $\delta_{\mathcal A_R}$ defines a Dirac measure of the "right" attractor, and conversely $\delta_{\mathcal A_R}$ a "left" attractor.
Then
\begin{equation}
    W_1\!\left(P(\cdot \mid s_1,a), P(\cdot \mid s_2,a)\right) = d(\mathcal A_L,\mathcal A_R),
\end{equation}
which stays bounded away from zero and does not reduce even as $\|s_1 - s_2\| \to 0$, where $d(\mathcal A_L,\mathcal A_R)$ denotes the distance between the two attractors in the underlying state space metric.
In this case $K_P$ is unbounded and Theorem \ref{theorem:smoothness_wasserstein} no longer applies.
Such behaviour occurs at tipping points, bifurcations, or in perfectly deterministic systems without noise.

\textbf{Expectation Based Learning Near Basin Boundaries.} When basin switching occurs, the scalar value function:
\begin{equation}
    V^\pi(s) = \mathbb{E}[Z^\pi \mid s_0 = s],
\end{equation}
can behave like a step function across the boundary. 
Small perturbations in state produce large jumps in the bootstrapped target, leading to large TD errors, unstable gradients, high sensitivity to initialisation and step size.
Moreover, expectation collapses multimodal outcomes.
If one branch yields $+100$ and the other $-100$, the learned value is $0$, which may correspond to no realisable trajectory of the system.

\textbf{How Distributional RL Fares.} Even when Assumption \ref{assumption:statistical_stability} fails, distributional RL still preserves multimodality. 
Near a separation point the return distribution is naturally bimodal, and distributional methods represent both modes rather than collapsing them into a single mean.
Further, the distributional Bellman operator involves only translation by $R(s,a)$ and scaling by $\gamma$:
\begin{equation}
    Z(s,a) \stackrel{D}{=} R(s,a) + \gamma Z(S',A').
\end{equation}
Both operations are non-expansive in Wasserstein distance: translation leaves the distance unchanged, while scaling by $\gamma$ contracts it:
\begin{equation}
    W_1(X+c,Y+c)=W_1(X,Y), \quad W_1(\gamma X,\gamma Y)=\gamma W_1(X,Y).
\end{equation}
Thus, discounting remains contractive even when transitions are discontinuous.
There are therefore two important regimes.
If $\gamma K_P < 1$, Theorem \ref{theorem:smoothness_wasserstein} guarantees Lipschitz smoothness of the return distribution in $W_1$.
If statistical stability fails (e.g. at bifurcations), this guarantee disappears. 
However, distributional RL still avoids expectation collapse and provides a better conditioned loss geometry than scalar TD learning.

\section{Ikeda Map One-Step Objective Figure}
\label{sec:appendix4}

\begin{figure}[htbp]
    \captionsetup{skip=0pt}
    \centering
        \includegraphics[width=1\linewidth]{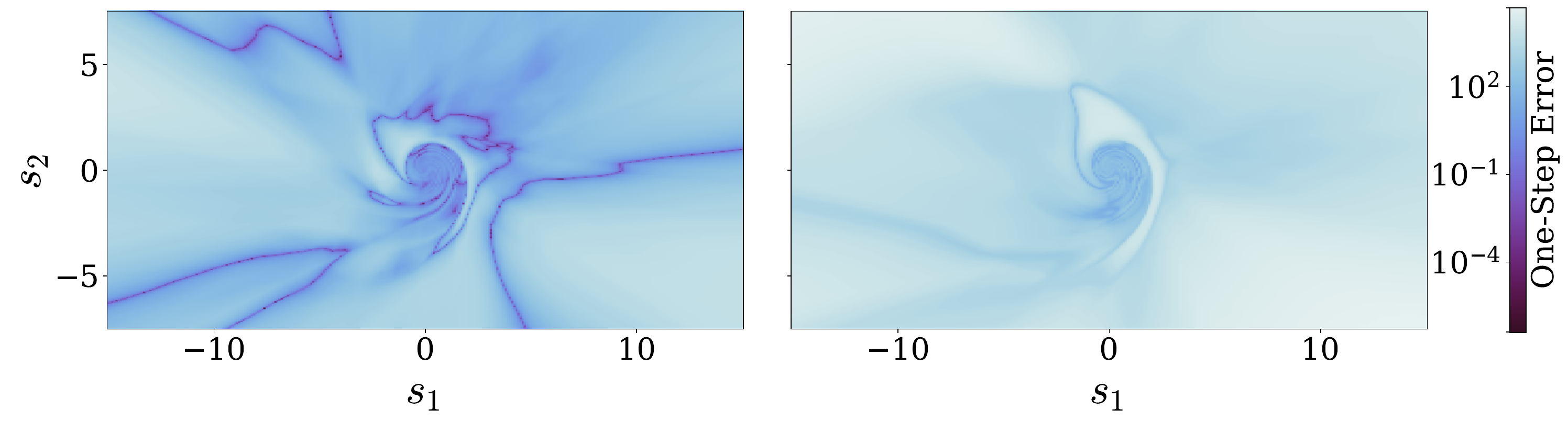}
    \caption{Ikeda Map objective landscapes. Left: DQN, Right: QRDQN. Figure data collected from agents after $3e8$ training steps.}
    \label{fig:ikeda_map_landscape}
    \vspace{-10pt}
\end{figure}

\section{QRDQN Distribution for Logistic Map}
\label{sec:appendix5}

To analyse the structure learned by QRDQN in the $1D$ environment, we visualise the full return distribution as a function of state. 
For a fixed (or greedy) action, the network outputs $N$ quantile locations $\{z_i(s)\}_{i=1}^N$ at quantile fractions $\tau_i = \frac{i+0.5}{N}$, which approximate the inverse cumulative distribution function (CDF) of the return $Z(s)$. 
From these quantile points we construct (i) an empirical CDF surface:
\begin{equation}
    F(z \mid s) \approx \frac{1}{N} \sum_{i=1}^{N} \mathbf{1}[z_i(s) \le z],
\end{equation}
and (ii) an approximate probability density function (PDF) surface obtained by histogramming the quantile locations at each state and normalising to unit mass. 
%
\begin{figure}[htbp]
    \captionsetup{skip=0pt}
    \centering
        \includegraphics[width=1\linewidth]{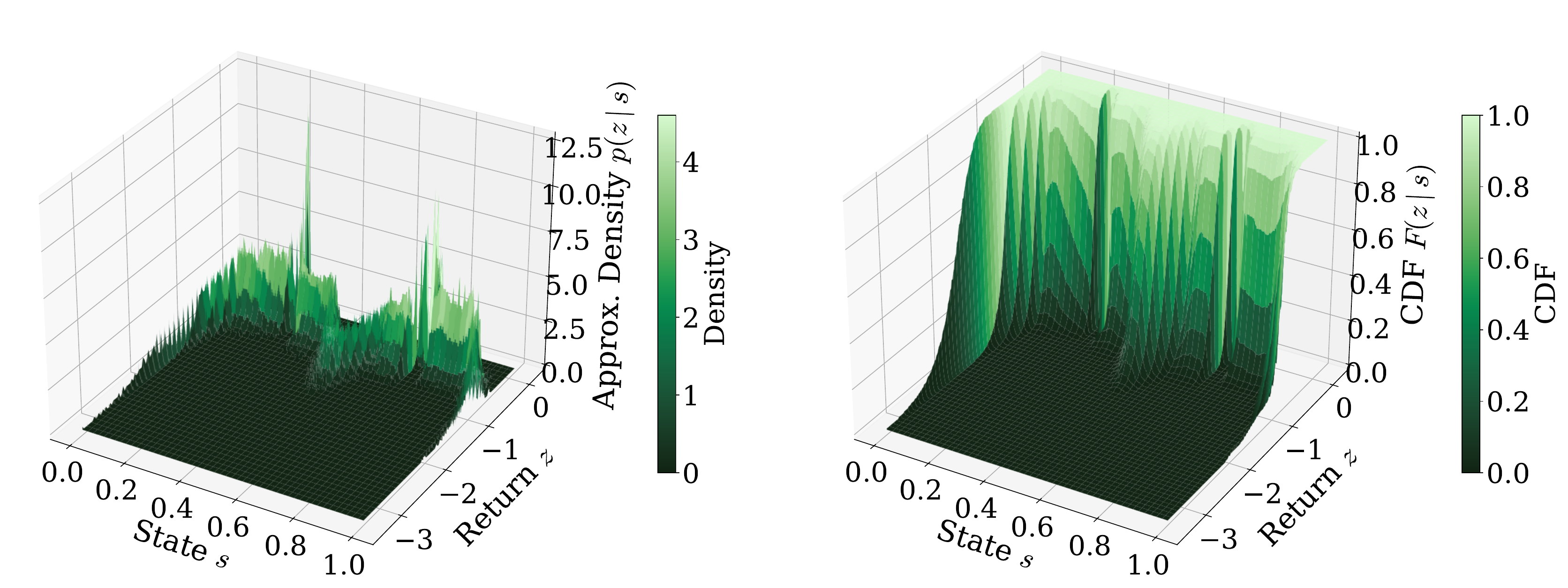}
    \caption{Left: QRDQN implied return PDF surface. Right: QRDQN implied return CDF surface.}
    \label{fig:cdf_pdf}
\end{figure}

Figure \ref{fig:cdf_pdf} visualises the return distribution learned by QRDQN for the Logistic Map as a function of state. 
In the PDF surface, probability mass concentrates in narrow ridges that shift smoothly across $s$, indicating that only a small number of trajectory outcomes dominate for each state while their relative likelihood evolves gradually. 
Regions with multiple peaks correspond to states where trajectories can follow qualitatively different regimes, producing multi-modal returns. 
The CDF surface shows steep transitions where many quantiles collapse around similar return values, revealing discrete bands of likely outcomes. Although these transitions are sharp in the return dimension, their location varies smoothly across state. 
Together, these patterns show that while individual trajectories diverge under chaotic dynamics, the return distribution evolves coherently across state, supporting our claim that modelling the full distribution captures structured variability that expectation based methods would collapse into a single value.

\section{Implementation Details}
\label{sec:implementation_details}

In this section, we outline the exact hyperparameters and network architectures steps used for our experiments to ensure full reproducibility. 
All algorithms were implemented using Jax \citep{bradbury2021jax}.

\subsection{Network Architectures}
\textbf{Value-based methods (DQN, QRDQN):} 
We use the same network for both.
The network consists of five hidden dense layers (feature sizes: $64, 128, 256, 128, 64$).
All hidden layers utilise ReLU activation functions. 
For DQN, the output layer consists of a single linear layer matching the action space $|\mathcal A|$. 
For QRDQN, the output layer dimension is $|\mathcal A| \times N$, where $N$ is the number of quantiles, representing the distribution of returns for each action.

\textbf{Actor-Critic method (PPO):} 
The actor and critic networks share the same hidden layer base consisting of two hidden dense layers (feature sizes: $256, 256$) but branch into separate fully connected heads. 
The policy head outputs a Gaussian distribution over actions, while the value head outputs a single scalar baseline.

\subsection{Hyperparameters}
Table~\ref{tab:hyperparams} summarises the hyperparameters used for DQN, QRDQN, and PPO. 

\begin{table}[h]
\centering
\caption{Hyperparameter configurations for DQN, QRDQN, and PPO.}
\label{tab:hyperparams}
\begin{tabular}{llc}
\toprule
\textbf{Algorithm} & \textbf{Hyperparameter} & \textbf{Value} \\
\midrule
\multirow{8}{*}{\textbf{Common / DQN}}
 & Optimiser & Adam \\
 & Learning rate & $1 \times 10^{-3}$ \\
 & Discount factor ($\gamma$) & $0.99$ \\
 & Replay buffer size & $1 \times 10^7$ \\
 & Batch size & $64$ \\
 & Target network update frequency & $64$ steps \\
 & Initial exploration $\epsilon$ & $1.0$ \\
 & Final exploration $\epsilon$ & $0.01$ \\
 & Exploration $\epsilon$ decay steps & $1 \times 10^6$ \\
\midrule
\multirow{3}{*}{\textbf{QRDQN}}
 & Number of quantiles ($N$) & $201$ \\
 & Huber loss threshold ($\kappa$) & $1.0$ \\
\midrule
\multirow{9}{*}{\textbf{PPO}}
 & Optimiser & Adam \\
 & Learning rate & $3 \times 10^{-4}$ \\
 & Number of parallel environments & $32$ \\
 & Rollout length per environment & $256$ \\
 & PPO clip parameter ($\epsilon$) & $0.2$ \\
 & GAE parameter ($\lambda$) & $0.95$ \\
 & Value loss coefficient & $0.5$ \\
 & Entropy coefficient & $0.01$ \\
 & Optimisation epochs & $4$ \\
\bottomrule
\end{tabular}
\end{table}



\end{document}